\documentclass[letterpaper]{article} 
\usepackage{aaai23}  
\usepackage{times}  
\usepackage{helvet}  
\usepackage{courier}  
\usepackage[hyphens]{url}  
\usepackage{graphicx} 
\def\UrlFont{\rm}  
\usepackage{natbib}  
\usepackage{caption} 
\usepackage{algorithm}
\usepackage{algorithmic}
\usepackage{amssymb}
\usepackage{amsmath}
\usepackage{amsthm}   
\usepackage{comment} 
\usepackage{xcolor}  
\usepackage{subcaption}  
\usepackage{enumitem}    
\usepackage{pgfplots}    
\usepgflibrary{plotmarks}

\usepackage{newfloat}
\usepackage{listings}
\DeclareMathOperator{\E}{\mathbb{E}}
\DeclareMathOperator{\IndicatorFunc}{\mathbf{1}}

\DeclareCaptionStyle{ruled}{labelfont=normalfont,labelsep=colon,strut=off} 
\floatstyle{ruled}
\newfloat{listing}{tb}{lst}{}
\floatname{listing}{Listing}
\title{Conditional Kernel Imitation Learning for Continuous State Environments}
\author{
    Rishabh Agrawal\textsuperscript{\rm 1},
    Nathan Dahlin\textsuperscript{\rm 2},
    Rahul Jain\textsuperscript{\rm 1},
    Ashutosh Nayyar\textsuperscript{\rm 1}
}
\affiliations{
    \textsuperscript{\rm 1} University of Southern California\\
    \textsuperscript{\rm 2} University at Albany, SUNY\\
    rishabha@usc.edu, ndahlin@albany.edu, \{rahul.jain, ashutosn\}@usc.edu

}

\usepackage{bibentry}

\newtheorem{theorem}{Theorem}
\newtheorem{lemma}[theorem]{Lemma}

\newcommand{\nate}[1]{{\textcolor{red}{Nate says: {#1}}}}
\newcommand{\rjain}[1]{{\textcolor{blue}{rjain says: {#1}}}}
\newcommand{\rishabh}[1]{{\textcolor{green}{rishabh says: {#1}}}}

\renewcommand{\hat}{\widehat}

\usepackage{thmtools}

\makeatletter
\newenvironment{customlegend}[1][]{%
    \begingroup
    \let\addlegendimage=\pgfplots@addlegendimage
    \let\addlegendentry=\pgfplots@addlegendentry
    \pgfplots@init@cleared@structures
    \pgfplotsset{#1}%
}{%
    \pgfplots@createlegend
    \endgroup
}%

\definecolor{bcColor}{RGB}{0, 255, 0}
\definecolor{vdiceColor}{RGB}{255, 0, 0}
\definecolor{rcalColor}{RGB}{110,117,14}
\definecolor{edmColor}{RGB}{0,255,255}
\definecolor{avrilColor}{RGB}{255,0,255}
\definecolor{dsfnColor}{RGB}{0, 0, 255}
\definecolor{ckilColor}{RGB}{94,60,153}
\definecolor{expertColor}{RGB}{240,100,10}
\definecolor{randomColor}{RGB}{0,0,0}

\makeatother

\pgfplotsset{compat=1.18}
\usepackage{dsfont}
\nocopyright

\begin{document}

\maketitle

\begin{abstract} Imitation Learning (IL) is an important paradigm within the broader reinforcement learning (RL) methodology. Unlike most of RL, it does not assume availability of reward-feedback. Reward inference and shaping are known to be difficult and error-prone methods particularly when the demonstration data comes from human experts. Classical methods such as behavioral cloning and inverse reinforcement learning are highly sensitive to estimation errors, a problem that is particularly acute in continuous state space problems. Meanwhile, state-of-the-art IL algorithms convert behavioral policy learning problems into distribution-matching problems which often require additional online interaction data to be effective. In this paper, we consider the problem of imitation learning in continuous state space environments based solely on observed behavior, without access to transition dynamics information, reward structure, or, most importantly, any additional interactions with the environment. Our approach is based on the Markov balance equation and introduces a novel conditional kernel density estimation-based imitation learning framework. It involves estimating the environment's transition dynamics using conditional kernel density estimators and seeks to satisfy the probabilistic balance equations for the environment. We establish that our estimators satisfy basic asymptotic consistency requirements. Through a series of numerical experiments on continuous state benchmark environments, we show consistently superior empirical performance over many state-of-the-art IL algorithms.

\end{abstract}

\section{Introduction}

Reinforcement Learning (RL) has produced a series of breakthroughs over the last decade from exceeding human proficiency at playing simple games such as in the Atari suite \cite{mnih2015human} to Go \cite{silver2016mastering} and StarCraft \cite{vinyals2019grandmaster}, and to protein structure prediction systems \cite{jumper2021highly}, etc. A fundamental premise that the `reward is enough’ \cite{silver2021reward} underlies all of such RL methodology. And yet, in most problems, a natural reward function is not available. Nor it may be possible to engineer one from intuition. Thus, a lot of effort is spent on reward shaping \cite{rewardShaping} to make RL algorithms work, often without success.

This problem is particularly acute when humans are in the loop, either as demonstrators or as evaluators. Often, demonstration data comes from human experts and it is impossible to infer precisely what reward function human experts really have in mind while taking actions. To be fair, several inverse RL (IRL) algorithms such as MaxEntropy-IRL \cite{maxEntRL} use  a methodology wherein a reward function is first inferred from the demonstration data, and then used in conjunction with RL algorithms to design near-optimal policies. This has two lacunae. First, the performance of the RL algorithms can be very sensitive to errors in the reward function estimate. And second, the expert may not be using a policy that is optimal with respect to any reward objective at all! There is thus a need to develop imitation learning algorithms that do not depend on reward inference as a first step \cite{aroraAndDoshi2021}.

Behavioral Cloning (BC) is a simple and natural idea for imitation learning \cite{firstBC}. It is a supervised learning scheme that aims to learn the expert policy as a map from states to actions. In fact, it largely ignores the inherent sequential nature of reinforcement learning problems. Unfortunately, it suffers from severe covariate shift issues as it fails to generalize to less-visited parts of the state space and ignores the sequential decision-making aspect. This results in propagation of errors in the agent's performance \cite{bcLimitation1} resulting in a limited practical ability to generalize effectively. 

To address the issue of compounding errors that can afflict methods like BC, IRL algorithms \cite{AndrewNG1, AndrewNG2, gameTheoreticAL, maxEntRL, nonLinearIRL} take a different approach by learning a reward function that takes into account that transitions come from trajectories. But this requires use of reinforcement learning algorithms, making them extremely expensive to run, and at the same time sensitive to reward function estimation errors. 

In contrast, adversarial Imitation Learning (AIL) is a technique centered around distribution matching through adversarial learning. It gained significant traction in the recent past as an approach to imitation learning \cite{gail, fu2018AIL, fDivAIL}. Within this framework, the objective transforms into finding a behavioral policy that minimizes the divergence between the target distribution and the distribution of state-action pairs generated by the behavioral policy during its interaction with the environment. The primary drawback of current distribution matching methods via AIL is that estimating distribution density ratios, a crucial step in each iteration, usually demands samples from the behavioral policy distribution. Consequently, new interactions with the environment are necessary for every behavioral policy update iteration. This limitation makes these algorithms unsuitable for problems where only offline data is available. This downside is even more apparent in continuous state and action space problems wherein each visited state is visited at most once, with most states not being visited at all in the demonstration data.

A related strand of literature on imitation learning \cite{BCMod1} assumes access to a generative model so trajectory data can be generated on the fly. We make no such assumption in our problem formulation.

In this paper, we aim to develop imitation learning algorithms that do not need reward-feedback, do not rely on distribution matching between occupation measures, are not doing behavioral-cloning but use the `meta'  knowledge that the underlying dynamics are Markovian, do not need access to a generative model, work for continuous state space problems, and allow for batch processing of the offline dataset.  This version of the imitation learning problem is relevant in many real-world decision-making applications, such as medical, healthcare, robotic manipulations, and autonomous vehicles \cite{ILSurveyAutonomous}, where experimentation is either costly or unsafe.

We introduce a simple and natural framework based on a simple premise. Namely, that the demonstration trajectory data satisfies the balance equation between the demonstration policy, the Markov decision process (MDP) transition density and that of the induced Markov chain. This allows us to incorporate the fact that the demonstration data is coming from a system with Markovian dynamics. The transition densities for the MDP and the Markov chain are then estimated using conditional kernel density estimators which are universal density estimators with proven asymptotic consistency properties. We start with the discrete state and action space setting, but then show that the framework extends to the continuous state space setting as well. We prove consistency properties of the estimators, and validate the algorithm in a series of continuous state space problems. Conceptual extension to continuous action space is straightforward but requires more work for numerical robustness.

\subsubsection*{Related Work.} We now discuss prior work broadly related to our paper. As already mentioned behavioral cloning faces a fundamental limitation due to  discarding distributional insights from the demonstrations \cite{bcLimitation1, bcLimitation2}. To address this,  several remedies have been suggested \cite{BCMod1, BCMod2} which involve either further online interactions with the environment, or the demonstrator,  or using  insights into model dynamics or the sparsity of rewards, all of which are in general impractical. The recent work \cite{discriminatorWeighted} aims to overcome these by using additional data from non-expert policies, which circumvents the need for additional online interactions but the additional offline data may not be available. The EDM approach \cite{SBIL} captures the expert's state occupancy measure by training an explicit energy-based model but its limitations has been scrutinized in \cite{sbilCritique}.

There also have been  efforts to further develop IRL approaches to overcome the limitations of earlier algorithms. \cite{klein2011,klein2012} introduce \textit{LSTD-$\mu$}, a temporal difference technique for calculating feature expectations. 
However, these approaches share the weaknesses of least squares estimators, being highly sensitive to basis features and training data distribution. \cite{lee2019} propose \textit{DSFN}, which estimates feature expectations in an off-policy setting. They also propose a transition-regularized imitation network that produces an initial policy close to expert behavior and an efficient feature representation. Despite these advancements, the assumption of complete knowledge about reward feature functions in these methods can often be unrealistic, especially for complex problems \cite{aroraAndDoshi2021}. \cite{bcLimitation2} introduced \textit{RCAL}, a non-parametric algorithm that employs a boosting method to minimize the criterion directly without feature selection steps and can help tackle some of the above issues. \cite{chan2021} propose \textit{AVRIL}, adopting a variational approach to jointly learn an approximate posterior distribution over reward and policy. 
However, due to inherent covariate shift problems, these methods encounter significant reward extrapolation errors, leading to misguided outcomes in novel environments and reduced learning efficiency. To address this, the \textit{CLARE} \cite{clare} model-based offline Inverse Reinforcement Learning (IRL) approach introduces conservatism to its estimated reward. It employs an IRL algorithm within an estimated dynamics model to learn the reward. 
However, limitations arise when dealing with a limited number of expert demonstrations or predominantly low-quality transition samples from a behavior policy. In such cases, forcing alignment of the empirical state-action visitation measure across all data can lead to a recovered reward or policy that mimics the suboptimal behavior policy, undermining the accuracy of the expert model \cite{Siliang}.

Adversarial Imitation Learning (AIL) approaches \cite{gail} were a breakthrough when they were introduced a few years ago \cite{AILMod1, AILMod2}. However, these approaches require online interactions with the environment, and thus are not applicable when we must work only with offline data.  Employing a distribution matching strategy, \cite{VDICE} introduces \textit{ValueDICE}, an offline objective for assessing the distribution ratio between the imitator and expert policies. Although theoretically allowing for comprehensive offline learning, the approach undertakes a complex alternating maximization-minimization optimization procedure. Additionally, it suffers from difficulty estimating the expectation of an exponential that introduces bias when approximating gradients using mini-batches \cite{SBIL}.

Thus, the algorithm we present in this paper is quite distinct in its approach from most of the prior literature. Additionally, it demonstrates promising preliminary empirical results.

\section{Preliminaries}\label{sec:prelim}

\subsection{The Imitation Learning Problem}

An infinite horizon discounted Markov decision process (MDP) $M$ is defined  by the tuple ($S,A,T,r,\gamma$) with states $s\in S$, actions $a\in A$ and successor states $s' \in S$ drawn from the transition function $T(s'\vert s, a)$. The reward function $r : S \times A \rightarrow \mathbb{R}$ maps state-action pairs to scalar rewards, and $\gamma$ is the discount factor. 
Policy $\pi$ is a probability distribution over actions conditioned on state and is given by
$\pi(a_{t}|s_{t}) = P_{\pi}(A_{t} = a_{t} | S_{t} = s_{t})$, where $a_{t}\in A$, $s_{t}\in S$, $\forall t = 0, 1, 2, \cdots$.
The induced occupancy measure of a policy is given as $\rho_{\pi}(s, a) := \E_{\pi}[\sum_{t=0}^{\infty} \gamma^{t} \IndicatorFunc_{s_{t} = s, a_{t} = a}]$, where the expectation is taken over $a_{t} \sim \pi(\cdot|s_{t})$, $s_{t+1} \sim T(\cdot|s_{t}, a_{t})$ for all $t$, and the initial state $s_0$. The corresponding state-only occupancy measure is given as $\rho_{\pi}(s) = \sum_{a} \rho_{\pi}(s,a)$.

In the imitation learning (IL) framework,
the agent is provided with trajectories generated by a demonstration policy $\pi_D$, collected as $D=\{(s_0,a_0), (s_1,a_1), (s_2,a_2),...\}$. 
The data $D$ does \emph{not} include the corresponding reward information $r_t$ at each time step. Indeed, rather than finding the optimal policy that maximizes the long-term reward, 
the IL objective is to learn a policy $\pi^{\star}$ that is close to $\pi_D$ in the following sense \cite{ImitationLearningPresentation}:
\begin{equation}
    \pi^{\star}\in \underset{\pi \in \Pi}{\arg\min}\,\mathbb{E}_{s \sim \rho_\pi} [\mathcal{L}(\pi(\cdot|s), \pi_D(\cdot|s))],
    \label{eq:ILObjective}
\end{equation}
where $\Pi$ is the set of all randomized (Markovian) stationary policies,
and $\mathcal{L}$ is a chosen loss function. In practice, \eqref{eq:ILObjective} can only be solved approximately in part due to the assumption that $\pi_D$ is unknown and only observed via the finite dataset $D$. 

\subsection{Conditional Kernel Density Estimation (CKDE)} \label{sec:ckde}

The imitation learning approach we will introduce will depend on transition density estimation. Though, statistical theory exists for it, conditional density estimation is a difficult problem due to lack of clarity on what parametric families of density functions are good candidates. Thus, we adopt the kernel density estimation framework \cite{kernelSmoothingBook} that are provably universal probability density estimators. 

We next outline the method in the case of two continuous scalar random variables, $X$ and $Y$ for the sake of simplicity. Let $f$ and $g$ denote the joint density of $(X, Y)$ and the marginal density of $X$, respectively.
The conditional distribution of $Y$, given $X$, is denoted as $h_{Y|X}(y|x) = f_{X, Y}(x, y)/g_{X}(x)$.   

Selecting a pair of kernel functions $K\,:\,\mathbb{R}\to\mathbb{R}$ and $K'\,:\,\mathbb{R}\to\mathbb{R}$ with respective scalar bandwidth parameters $h>0$ and $h'>0$ and given a set of $n$ samples $\{(x_{i} , y_{i} )\}_{i=1}^{n}$, the kernel density estimation (KDE) approximations $\hat{f}$ and $\hat{g}$ for the joint and marginal distributions, respectively, are obtained as follows:
\begin{equation}
    \begin{gathered}
        \hat{f}_{X, Y}(x, y) = \frac{1}{n} \sum_{i=1}^{n} \frac{1}{h}K\bigg(\frac{x - x_i}{h}\bigg)\frac{1}{h'}K'\bigg(\frac{y - y_i}{h'}\bigg),
        \\
        \hat{g}_{X}(x) = \frac{1}{n} \sum_{i=1}^{n} \frac{1}{h}K\bigg(\frac{x - x_i}{h}\bigg).
    \end{gathered}
    \label{eqn:IndividualKernel}
\end{equation}
Using the approximations in \eqref{eqn:IndividualKernel}, the approximate conditional density $\hat{h}_{Y|X}$ can be computed as
\begin{equation}
    \hat{h}_{Y|X}(y|x) = \frac{\hat{f}_{X, Y}(x, y)}{\hat{g}_{X}(x)}.
    \label{eqn:ConditionalKernelBasicEqn}
\end{equation}



In more general cases involving random vectors, analogous estimates to those in \eqref{eqn:IndividualKernel} and \eqref{eqn:ConditionalKernelBasicEqn} may be obtained using kernel functions defined according to 
\begin{equation}\label{eqn:K_H_def}
    K_{H}(x) = |H|^{-\frac{1}{2}} K(H^{-\frac{1}{2}} x),   
\end{equation}
where $H$ is a symmetric positive definite \textit{bandwidth matrix} of appropriate dimension, $m$, with determinant $|H|$, and $K$ is a real-valued function satisfying $\int_{\mathbb{R}^m} K(x)dx = 1$. For example, the KDE estimate for the marginal distribution of random vector $X$ is defined as
\begin{equation}
    \hat{g}_{X}(x; H) = \frac{1}{n} \sum_{i=1}^{n} K_{H}\big(x - x_i\big).
    \label{eqn:MultivariateConditionalKernel}
\end{equation}



An example of such a multivariate kernel function is the standard $m$-variate normal density function
\[
K(x) : = (2\pi)^{-\frac{m}{2}} \exp\left(-\frac{x^{T}x}{2}\right).
\]

For more details on kernel density estimation, see \cite{multivariateKernelBook}. 

\section{Conditional Kernel Imitation Learning}\label{sec:algo}


We next describe our imitation learning algorithm. A key premise of our algorithm is that the demonstration trajectories from an expert must satisfy the Markov balance equation under the demonstrator's policy $\pi_D$, and thus we must use that to guide the agent's learning. The use of the balance equation then requires estimation of certain transition (conditional probability) density functions which we do via conditional kernel density estimation methods. Then, the problem reduces to identifying policies that best fit the balance equation. We elucidate this procedure in Algorithm \ref{alg:CKIL}, prove its theoretical properties and then present numerical evidence of its efficacy on a number of benchmark environments.


\subsection{The Markov Balance Equation}

Consider a demonstration policy $\pi_D$ that is used to take actions starting from an initial state $s_0$. Let $T(s'|s,a)$ denote the transition density function of the MDP. Note that $\pi_D(a|s)$ is a randomized Markovian, stationary policy that denotes the probability of taking action $a$ in state $s$. This will induce a Markov chain on the state space $S$. Let its transition density be denoted $P(s'|s)$. Then, the Markov balance equation is given by
\[
P(s'|s) = \sum_a \pi_D(a|s)T(s'|s,a).
\]
Unfortunately, this involves a sum, and hence is difficult to use. We thus, use the following alternate form which is a transition density of the induced Markov chain on the state-action space,
\begin{equation}
    P_{\pi_D}(s', a'|s,a) = \pi_D(a'|s')T(s'|s,a).
    \label{eqn:policy_balance}
\end{equation}

The above balance equation is the basis of our imitation learning approach. If we can estimate $P_{\pi_D}$ and $T$ in \eqref{eqn:policy_balance} (estimates denoted by $\hat{P}$ and $\hat{T}$ respectively), we can then infer a policy $\pi_D$ that satisfies it. Unfortunately, the problem is ill-conditioned, and we will need to impose additional criterion such as a regularization term.

We consider a parametric class of policies for the agent parametrized by $\theta$  and formulate the following optimization problem: 
\begin{equation}
    \begin{gathered}
        \min_{\theta \in \Theta} 
        \int_{(s',a')} \int_{(s,a)} \big[\hat{P}(s',a'|s,a) - \pi_{\theta}(a'|s')\hat{T}(s'|s,a)\big]^{2}
        \\ \,d\mu(s,a) \,d\mu(s',a') 
        - \lambda\int_{s'} H(\pi_{\theta}(\cdot|s')) \,d\nu(s').
    \end{gathered}
    \label{eqn:kernelRegularizedObjective}
\end{equation}

In \eqref{eqn:kernelRegularizedObjective}, the first term ensures that the balance equation is satisfied approximately. The second term involves $H(\pi_{\theta}(\cdot|s'))$, which is the entropy of the probability distribution $\pi_{\theta}(\cdot|s')$ on actions when the state is $s'$. It penalizes less randomized policies in favor of highly randomized policies. $\lambda \geq 0$ is a regularization parameter that governs relative weight on the first and second terms. Here, $\mu$ and $\nu$ denote reference measures on state-action pairs and states respectively. For example, they can be the counting measures obtained from the dataset. $\Theta$ is a given parameter set. The parameters could be weights of a neural network, for example.


\begin{algorithm*}[tb]
\caption{Conditional Kernel Imitation Learning (CKIL)}
\label{alg:CKIL}
\textbf{Input}: Expert dataset of trajectories $D$ = $\{(s_{i} , a_{i} )\}_{i=1}^{n}$\\
\textbf{Output}: $\theta^{*}$
\begin{algorithmic}[1] 
\STATE Initialize policy parameter $\theta$\\
\STATE Transform dataset $D$ into $(s, a, s', a')$ tuples, then store them in buffer $B$.
\STATE Obtain $\hat{P}, \hat{T}$ in \eqref{eqn:kernelEstimators} via CKDE on $B$\\
\FOR{ $iter = 0, 1, \hdots$}
\STATE Sample a batch $b_{iter}$ of $(s, a, s', a')$ tuples from $B$
\STATE Calculate empirical estimate of the objective function in \eqref{eqn:kernelRegularizedObjective} using all $(s, a, s', a')$ $\in b_{iter}$ as:
\begin{equation}
    \sum_{(s', a')} \sum_{(s, a)} \big[ \hat{P}(s', a'|s, a) - \pi_{\theta}(a'|s') \hat{T}(s'|s,a)\big]^{2} + \lambda \sum_{s'} \sum_{a'} \pi_{\theta}(a'|s') log(\pi_{\theta}(a'|s'))
    \label{empiricalObjectiveFunction}
\end{equation}
\STATE Update the policy parameter $\theta$ using gradient update to minimize the calculated empirical estimate of the objective function
\ENDFOR
\STATE \textbf{return} $\theta^{*}$
\end{algorithmic}
\end{algorithm*}

\subsection{Transition Density Estimation} \label{sec:TDE}

We now discuss how to use kernel density estimation methods for estimating the two conditional densities $P_{\pi_{D}}$ and $T$. We first discuss the discrete state and action space setting where the form of estimates is simple and intuitive. We then discuss the continuous setup.

\subsubsection{Discrete Spaces.}
For environments where both the state and action spaces are discrete, the estimates $\hat{P}$ and $\hat{T}$
can be calculated as:
\begin{equation}
    \begin{gathered}
    \hat{T}(s'|s,a):= \frac{\eta (s, a,s' )}{\eta (s, a)},~~\text{and}~~
    \\
    \hat{P}(s', a'|s, a):= \frac{\eta (s, a, s', a')}{\eta (s, a)}
    \end{gathered}
    \label{eqn:discreteProbabilites}
\end{equation}
where $\eta$ denotes the counting measure, i.e., the number of times a given tuple or sequence appears in the dataset $D$. 
If the counting measures in the denominator are zero when that state-action pair is not visited at all in the given dataset, it implies that we have no information about transitions from it. In such cases, we will take the conditional density to be uniform. 

\subsubsection{Continuous Spaces.}
Estimation of transition densities in this setting is more challenging since no visited state would appear twice, and most of the states would never be visited in any given dataset. This, thus calls for conditional density estimation using more sophisticated methods. 

These include a range of techniques such as parametric approaches like mixture density network \cite{MDN}, normalizing flows \cite{normalizingFlows}; non-parametric methods like Gaussian process conditional density estimation \cite{ckdeAlternative3}, CKDE \cite{NonparametricEconometrics}; and semi-parametric methods like least squares conditional density estimation \cite{LSCKDE}. In this study, we opt for using CKDE since it  is a closed-form, non-parametric method that can be easily implemented and adapted to different data types. Further, CKDE  provides a consistent estimator under appropriate conditions \cite{multivariateKernelBook}.

As described in Section \ref{sec:ckde}, the kernel functions use a  difference between two samples/values (e.g., $x-x_{i}$) as their argument \eqref{eqn:MultivariateConditionalKernel}. This difference can be alternatively replaced by a suitable distance metric, as indicated in prior work \cite{distanceMetric}.
We define three distinct distance metrics: one to assess the dissimilarity between (next state, next action) pairs, another for (state, action) pairs, and a final one for next states. These metrics are denoted as $d_{1} : (S \times A) \times (S \times A) \rightarrow \mathbb{R_{+}}$, $d_{2} : (S \times A) \times (S \times A) \rightarrow \mathbb{R_{+}}$, and $d_{3} : S \times S \rightarrow \mathbb{R_{+}}$ respectively. Similarly, we define $H_1$, $H_2$, and $H_3$ as bandwidth matrices for the kernels $K_{H_1}$, $K_{H_2}$, and $K_{H_3}$, respectively. $H_1$, $H_2$, and $H_3$ are square matrices with dimensions matching those of the $(s', a')$ pair, $(s, a)$ pair, and $s'$, respectively. The CKDE approximations $\hat{P}$ and $\hat{T}$
are then computed as 
\begin{equation}
    \begin{gathered}
    \hspace*{-5.3cm}   \hat{P}(s',a'|s,a) = \\
    \frac{\sum_{l=1}^{n} K_{H_{1}}\big(d_{1}((s', a'), (s_l', a_l'))\big)  K_{H_{2}}\big(d_{2}((s, a), (s_l, a_l))\big)}{\sum_{l=1}^{n} K_{H_{2}}\big(d_{2}((s, a), (s_l, a_l))\big)},
    \vspace{0.5cm}
    \\
    \hspace*{-6cm}   ~\text{and}~~\hat{T}(s'|s,a) = \\
    \frac{\sum_{l=1}^{n} K_{H_{3}}\big(d_{3}(s', s_l')\big)  K_{H_{2}}\big(d_{2}((s, a), (s_l, a_l))\big)}{\sum_{l=1}^{n} K_{H_{2}}\big(d_{2}((s, a), (s_l, a_l))\big)}.
    \end{gathered}
    \label{eqn:kernelEstimators}
\end{equation}

We combine the transition estimation procedures of \eqref{eqn:discreteProbabilites} and \eqref{eqn:kernelEstimators} with the balance equation based optimization problem in \eqref{eqn:kernelRegularizedObjective} in our conditional kernel imitation learning (CKIL) algorithm whose pseudo-code is presented in Algorithm \ref{alg:CKIL}.

\subsection{Theoretical Guarantees}\label{subsec:Theory}

In this section, we focus on the conditional kernel density estimation part of our approach and show that, as the training dataset size $n$ approaches infinity, the CKDE estimates in \eqref{eqn:kernelEstimators} converge towards the corresponding true conditional distributions. 



Here we introduce the kernel functions $K_i$, $i\in\{1,2,3\}$, so that using \eqref{eqn:K_H_def}, the kernels $K_{H_i}$ for $i\in\{1,2,3\}$ appearing in \eqref{eqn:kernelEstimators}
\begin{equation}\label{eqn:K_i_def}
    K_{H_i}(x) = |H_i|^{-\frac{1}{2}} K_i(H^{-\frac{1}{2}} x),
\end{equation}
where $x$ is of appropriate dimension.
Our theoretical guarantee holds under the following assumptions \cite{multivariateKernelBook}:

\begin{description}
    \item[(A1)] Suppose the buffer $B$ in Algorithm \ref{alg:CKIL} consists of $n$ iid tuples $(s,a,s',a')$ generated according to a probability distribution $P(s,a,s',a') = \mu(s,a)P_{\pi_D}(s', a'|s,a)$, where $P_{\pi_D}$ is the  transition probability density of the induced Markov chain on the state-action space under the demonstration policy $\pi_D$ (see \eqref{eqn:policy_balance}) and $\mu$ is a reference measure on $(s,a)$. Further, $P$ has a  density function $g$ that is
   square-integrable and twice differentiable, with all of its second-order partial derivatives bounded, continuous and square integrable. Also assume that the marginals $P(s', s, a)$ and $P(s, a)$ satisfy these properties.
    \item[(A2)] The kernels $K_i$ for $i\in\{1,2,3\}$ in \eqref{eqn:K_i_def} are square integrable, zero-mean, spherically symmetric, and with common finite second-order moment $\int_{\mathbb{R}^{m_i}}zz^TK_i(z)dz=\sigma^2I_{m_i}$. 

    \item[(A3)] For each kernel $K_{H_i}$ as defined in \eqref{eqn:K_H_def}, the bandwidth matrices $H_i(n)$ (where $n$ is the number of tuples in $B$) form a sequence of positive definite, symmetric matrices such that  $H_i(n) \rightarrow 0$ and $n^{-1/2}|{H_i(n)}|^{-1 / 2} \rightarrow 0$ as $n \rightarrow \infty$. 
\end{description}

\begin{theorem}
    Suppose assumptions (A1)-(A3) are true. Let $\widehat{P}_n$ and $\widehat{T}_n$ be the CKDE estimates constructed using \eqref{eqn:kernelEstimators} and a buffer $B$ with $n$ tuples. Then, for each $(s,a,s',a')$, as $n \rightarrow \infty$,
    \begin{align}
           &\widehat{P}_n(s', a' | s, a) {\stackrel{P}{\parbox{1.0cm}{\rightarrowfill}}} P_{\pi_D}(s', a' | s, a),
            \notag \\
            &\widehat{T}_n(s'| s, a) {\stackrel{P}{\parbox{1.0cm}{\rightarrowfill}}} T(s'| s, a).
     \end{align}
    \label{thm1}
\end{theorem}

\begin{proof}


Let $\hat{f}(\cdot)$ and $\hat{g}(\cdot)$  represent the numerator and denominator of $\hat{T}(s'|s,a)$ respectively from \eqref{eqn:kernelEstimators}, i.e.,
    \begin{equation}
    \begin{gathered}
        \hat{f}(s', s, a) = \sum_{l=1}^{n} K_{H_{3}}\big(d_{3}(s', s_l')\big)  K_{H_{2}}\big(d_{2}((s, a), (s_l, a_l))\big),
        \\
        \hat{g}(s, a) = \sum_{l=1}^{n} K_{H_{2}}\big(d_{2}((s, a), (s_l, a_l))\big).
        \end{gathered}
        \label{fAndGEstimate}
    \end{equation}
    $\hat{f}(\cdot)$ and $\hat{g}(\cdot)$ are, in fact, the KDEs of $P(s,a,s')$ and $P(s,a)$, where  $P(s,a,s')$ and $P(s,a)$ are marginals of $P(s,a,s',a')$ mentioned in Assumption (A1). 

    We now show why Lemma \ref{thm2}, given in Appendix, can be adopted to prove this theorem under assumptions (A1)-(A3). We argue as follows:

\begin{enumerate}
    \item We assume (A1) that $P(s, a, s', a')$ has a  density function $g$ that is square-integrable and twice differentiable, with all of its second-order partial derivatives bounded, continuous and square integrable and so does its marginals $P(s', s, a)$ and $P(s, a)$. This leads to the satisfaction of condition (C1).

    \item From assumption (A2), $\int_{\mathbb{R}^{m_i}}zK_{i}(z) dz=0$ for $i = \{2,3\}$, where $z_i$ is a vector of size $m_i$. Partition the vector $z$ as $z=[z_3,z_2]$ and let $m=m_2+m_3$ and $K(z) = K_3(z_3)K_2(z_2)$.
    Then for $t \leq m_3$, 
   \begin{equation}
    \begin{split}
        &\int_{\mathbb{R}^m} z_{t} K(z) dz= \int_{\mathbb{R}^m} z_{t} K_{3}(z_3) K_{2}(z_2) dz\\
        &\quad= \int_{\mathbb{R}^{m_2}} K_{2}(z_2) dz_2\int_{\mathbb{R}^{m_3}}z_{t} K_3(z_3)dz_3 \\
        &\quad= \int_{\mathbb{R}^{m_3}}z_{t} K_3(z_3)dz_3=0,
    \end{split}
   \end{equation}
    which follows from (A2). This can be shown for any $t \in \{1, 2, \hdots, m\}$. Hence, $\int_{\mathbb{R}^m} z K(z) dz=0$ is satisfied corresponding to condition (C2). 
    
    Now,
    \begin{equation*}
    \begin{split}
        &\int_{\mathbb{R}^m} zz^{T} K(z) dz\\
        &\qquad=
        \int_{\mathbb{R}^m}\begin{bmatrix}
        z_3z_3^T  & z_3z_2^T\\
        z_2z_3^T  & z_2z_2^T\\
        \end{bmatrix}
        K_{3}(z_3) K_{2}(z_2)
    dz_3  dz_2\\
    &\qquad= \sigma^2 \begin{bmatrix}
        I_{m_3}  & 0\\
        0  & I_{m_2}
    \end{bmatrix} = \sigma^2I_m.
    \end{split}
    \end{equation*}
Hence, $K(z) = K_3(z_3)K_2(z_2)$ satisfies condition (C2).
    
    \item Consider $H(n)$ to be a block diagonal matrix with $H_3(n)$ and $H_{2}(n)$ as the two block diagonal entries with $H_3(n)$ and $H_2(n)$ satisfying assumption (A3). Then the matrices $H(n)$ form a sequence of positive definite, symmetric matrices. Using \eqref{eqn:MultivariateConditionalKernel} with this $H$, the kernel estimate for $P(s', s, a)$ takes the product kernel form as seen for $\hat{f}(\cdot)$ in \eqref{fAndGEstimate}. Now, $|H(n)| = |H_3(n)| |H_{2}(n)|$, this implies that as $n \rightarrow \infty$, $n^{-1}|H(n)|^{-1 / 2} \rightarrow 0$ because $n^{-1/2}|H_{i}(n)|^{-1 / 2} \rightarrow 0$ for $i = \{2,3\}$. Also, vec $H(n) \rightarrow 0$ as vec $H_{i}(n) \rightarrow 0$ for $i = \{2,3\}$. Therefore, condition (C3) is satisfied.

\end{enumerate}
    
    Having satisfied conditions (C1)-(C3), we may apply the argument found in Sections 2.6-2.9 of \cite{multivariateKernelBook} and conclude that 
    \begin{equation*}
        \begin{gathered}
            \hat{f}(s', s, a) {\stackrel{P}{\parbox{1.0cm}{\rightarrowfill}}} P(s', s, a),
            \\
            \hat{g}(s, a) {\stackrel{P}{\parbox{1.0cm}{\rightarrowfill}}} P(s, a).
        \end{gathered}
    \end{equation*}

    It follows from the Continuous Mapping Theorem \cite{mann1943stochastic} that taking the ratio of $\hat{f}$ and 
    $\hat{g}$ produces a consistent estimator of $\frac{P(s', s, a)}{P( s, a)} = T(s'|s,a)$. That is,
    $\widehat{T}_n(s'| s, a) = \frac{\hat{f}(s', s, a)}{\hat{g}(s, a)} {\stackrel{P}{\parbox{1.0cm}{\rightarrowfill}}}  T(s'|s,a)$.\\


    A similar approach can be used for establishing the asymptotic convergence in probability for the CKDE of $P_{\pi_{D}}(s', a'|s, a)$.

\end{proof}

\begin{figure}[t]
\centering
\includegraphics[width=0.9\columnwidth]{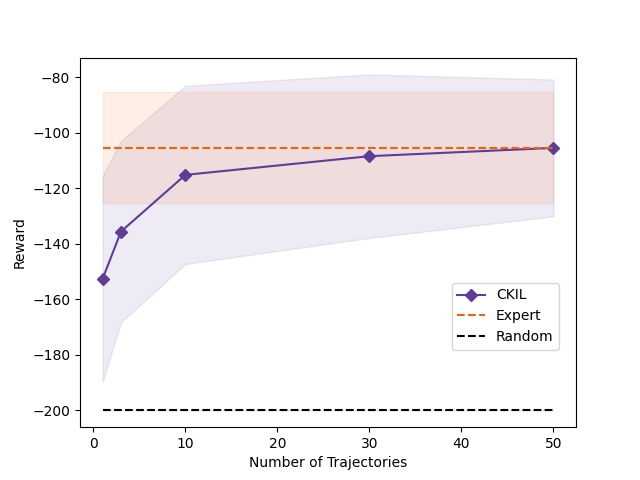} 
\caption{Average rewards attained by the CKIL agent in a discretized MountainCar environment for a varying number of trajectories (higher values indicate better performance).}
\label{fig1}
\end{figure}

\section{Experimental Results}
\label{sec:expts}

\subsection{Experimental Setup}

We now validate the empirical performance of the CKIL algorithm on a number of benchmark environments from Open AI Gym \cite{openAIGym}.  These encompass environments of varying complexities used in the RL literature. These include the MountainCar environment, where the goal is to reach the top of the mountain \cite{mountainCar}; CartPole, which aims to balance a pendulum on a frictionless track \cite{cartpole}; Acrobot, which aims to swing limbs around a system of joints to achieve a specified height \cite{acrobot}; and LunarLander, which aims to optimize a rocket's trajectory to achieve a successful landing \cite{LunarLander}. 

To create a demonstration dataset $\emph{D}$, we generate data using pre-trained and hyperparameter-optimized agents available in the RL Baselines Zoo \cite{rl-zoo3}. In particular, we used a PPO agent for LunarLander-v2, a DQN agent for CartPole-v1 and an A2C agent for Acrobot-v1. 

\paragraph*{Benchmark Algorithms.}
We compare the performance of our CKIL algorithm (Algorithm \ref{alg:CKIL}), with a range of offline IRL/IL/AIL benchmarks. This comprehensive assessment covers a spectrum of methodologies, including the inherently offline Behavioral Cloning (BC); ValueDICE (VDICE), a sample-efficient AIL approach designed for offline scenarios by removing replay regularization; reward-regularized classification (RCAL), a large margin classification approach, which introduces a sparsity-based penalty on inferred rewards to exploit dynamics information; Energy-based Distribution Matching (EDM), an offline imitation learning algorithm that captures the expert's state occupancy patterns through explicit training of an energy-based model; AVRIL, a recent model-free offline IRL technique employing a variational approach to simultaneously learn an approximate posterior distribution over rewards and policies; and Deep Successor Feature Network (DSFN), an offline adaptation of the max-margin IRL algorithm that transcends linear approaches by introducing a deep network architecture and employing least-squares temporal-difference learning to produce both reward and policy outputs.

\begin{figure*} 
    \centering
        \begin{tikzpicture}
    \begin{customlegend}[legend columns=-1]
    \addlegendimage{bcColor,mark=*, thick, mark options={solid,scale=1.5}}
    \addlegendentry{BC}
    \addlegendimage{vdiceColor,mark=triangle*,thick, mark options={solid,scale=1.5}}
    \addlegendentry{VDICE}
    \addlegendimage{rcalColor,mark=|, thick, mark options={solid,scale=1.5}}
    \addlegendentry{RCAL}
    \addlegendimage{edmColor,mark=asterisk, thick, mark options={solid,scale=1.5}}
    \addlegendentry{EDM}
    \addlegendimage{avrilColor,mark=x, thick, mark options={solid,scale=1.5}}
    \addlegendentry{AVRIL}
    \addlegendimage{dsfnColor,mark=triangle*, thick, mark options={solid,scale=1.5,rotate=180}}
    \addlegendentry{DSFN}
    \addlegendimage{ckilColor,mark=diamond*, thick, mark options={solid,scale=1.5}}
    \addlegendentry{CKIL}
    \addlegendimage{expertColor,mark=|*, densely dashed, thick}
    \addlegendentry{Expert}
    \addlegendimage{randomColor,mark=|*, densely dashed, thick}
    \addlegendentry{Random}
    \end{customlegend}
    \end{tikzpicture}
    \begin{subfigure}[b]{0.33\textwidth}
        \includegraphics[width=\linewidth]
        {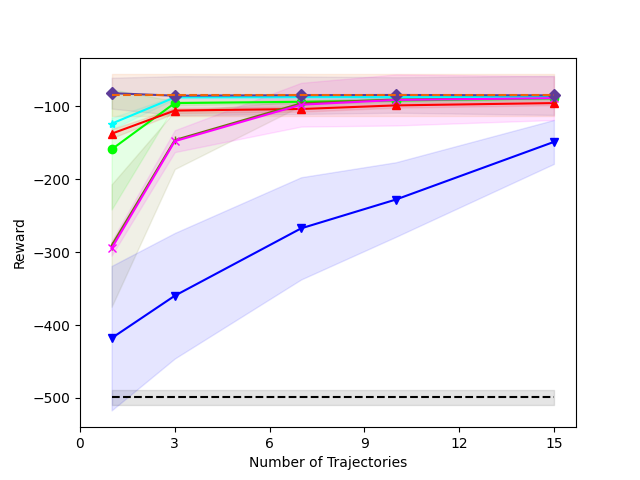}
        \caption{Acrobot}
    \end{subfigure}
    \hfill
    \begin{subfigure}[b]{0.33\textwidth}
        \includegraphics[width=\linewidth]
        {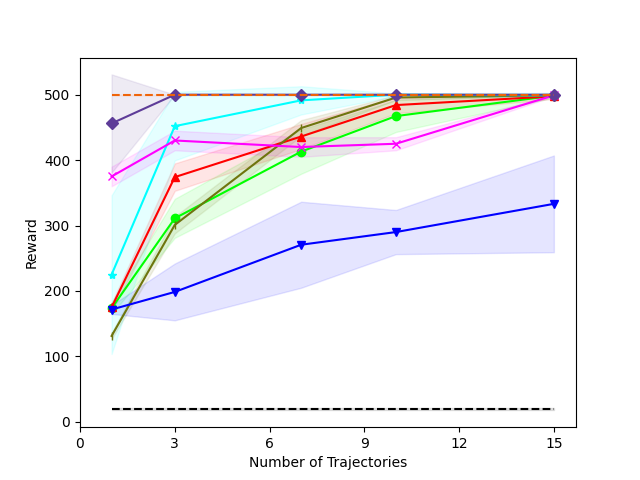}
        \caption{CartPole}
    \end{subfigure}
    \hfill
    \begin{subfigure}[b]{0.33\textwidth}
        \includegraphics[width=\linewidth]
        {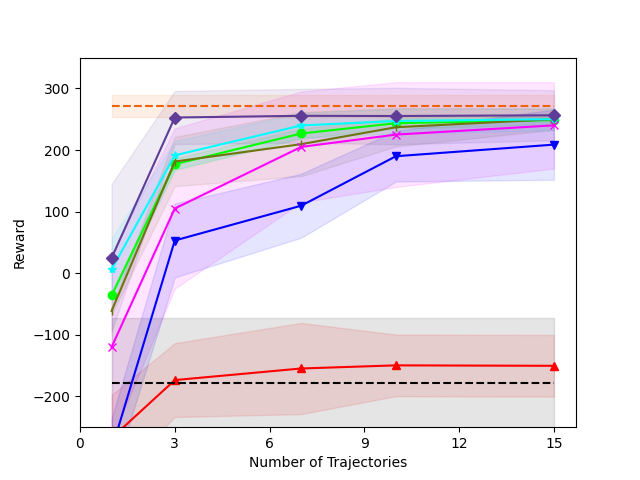}
        \caption{LunarLander}
    \end{subfigure}

    \caption{Average rewards achieved by benchmark IRL/IL/AIL and CKIL policies during real-time deployment plotted against the number of trajectories included in demonstration dataset $D$ (higher values indicate better performance).}
    \label{fig:gymplots}
\end{figure*}

\paragraph*{Implementation.}
The policy $\pi_{\theta}$ in \eqref{empiricalObjectiveFunction} is embodied by a neural network (NN) architecture. This NN comprises two hidden layers featuring the Rectified Linear Unit (ReLU) activation function. The final layer employs a softmax function to produce a probability distribution over actions when given a state as an input. To facilitate comparison, all benchmarks adopt a common neural network architecture consisting of two hidden layers comprising 64 units each, with Exponential Linear Unit (ELU) activation functions. Training is carried out using the Adam optimizer \cite{adam} with individually tuned learning rates.

In the case of VDICE, we used the publicly available source code provided at \cite{vdiceGithub}. It is worth noting that, for VDICE, offline learning is achieved by configuring the ``replay regularization" coefficient to zero. Our execution of EDM leveraged the source code accessible at \cite{edmGithub}. It is essential to highlight that the contrast between BC and EDM predominantly stems from the introduction of $L_{\rho}$, an occupancy loss defined in the EDM work, while deriving the RCAL loss is a straightforward process involving the inversion of the Bellman equation. As for AVRIL and DSFN, the applicable source codes are accessible at \cite{avrilGithub}, \cite{dsfnGithub} respectively. Additional specifics regarding the experimental setup and benchmarks implementations are available in the Appendix.

\paragraph*{Choice of Kernel.} We utilize the Gaussian kernel due to its universal nature, i.e, its capacity to uniformly approximate arbitrary continuous target functions on any compact subset of the input space \cite{universalGaussian}. 
Moreover, its properties such as square integrability, finite second-order moment, and spherical symmetry are desirable properties for density estimation, facilitating the asymptotic convergence of the conditional kernel density estimator to the true conditional density. 

We consider a Euclidean distance metric for $d_{1}$, $d_{2}$, and $d_{3}$ and utilize a diagonal bandwidth matrix with the same values across its diagonal elements. These matrices can then be denoted as $H_{i} = h_{i}I_{m_{i}}$, where $m_{i}$ is the corresponding appropriate dimension for $i = 1, 2, 3$. Here, each $h_{i}$ is treated as a hyper-parameter and we tuned these in our experimental work depending on what yielded better performance. 
We would like to emphasize that in prior research \cite{Do-validation, Scholkopf2002, Silverman86, MixingTheEstimators},
approaches for systematic selection of bandwidth parameters, which should decrease as the dataset size grows, have been developed. These methods can be applied to more intricate problems where manual tuning is impractical. Specific values of $h_{i}$ employed for various experiments are detailed in the Appendix.

\subsection{Results}

\subsubsection{Discrete States and Actions.} We first consider the discretized mountain car problem \cite{mountainCar} as  an example of a  discrete state and action space environment. 
To that end, we transformed the original continuous $2$-dimensional state space of the MountainCar environment into a grid configuration measuring $15$ by $15$. Subsequently, we  estimate $\hat{P}$ and $\hat{T}$ using equation \eqref{eqn:discreteProbabilites}. The dataset $D$ is generated using the policy outlined in \cite{mountainCarImplementation}. We begin with one trajectory in $D$, and increase to $50$ trajectories, with observations summarized in Figure \ref{fig1}. With more data, we see that the CKIL agent's performance improves and  achieves expert-level proficiency with 50 trajectories.


\subsubsection{Continuous States.}
We next address our main goal, namely imitation learning for continuous state environments. When provided ample demonstration data, all benchmarks exhibit the capability to attain performance comparable to the demonstrator level. Thus, we evaluate the algorithms' capacity to manage sample complexity in scenarios with limited data. To achieve this, we assess their performance while granting access to a specific number of trajectories, which we vary. This setup mirrors the configuration described in \cite{SBIL}. We vary the amount of demonstration data $\emph{D}$, ranging from a single trajectory to 15, to illustrate the sample complexity concept. 

The algorithms used a batch of 1, 3, 7, 10, or 15 trajectories, each uniformly drawn from a pool of 1000 expert trajectories. Subsequently, each algorithm underwent a training phase until convergence. For testing, 300 live roll-outs were conducted in the simulated environment, during which the average accumulated reward for each episode was measured. This entire process was repeated 10 times, utilizing diverse initialization and observed trajectories in each iteration.

Figure \ref{fig:gymplots} illustrates the average rewards for policies learned by different algorithms with increasing numbers of demonstration trajectories. Across all tasks, the results showcase CKIL's capability to learn effective policies, manifesting robust and consistently superior performance compared to other well-known algorithms, especially when data is scarce. 

Remarkably, even with a mere single expert trajectory, CKIL manages to attain expert-level performance in the Acrobot environment, and it approaches expert-level performance in the CartPole environment. Moreover, employing only three trajectories enables it to achieve expert-level performance on the CartPole environment.
In the LunarLander environment, a substantial performance enhancement becomes evident within $3$ trajectories. At this juncture, the agent's performance nearly matches that of an expert, surpassing all benchmark algorithms by a considerable margin.

Additionally, it is worth highlighting that within the confines of this exclusively batch-oriented context, the off-policy adaptations of online algorithms (VDICE, DSFN) do not exhibit the same degree of consistent performance as their inherently offline counterparts. This highlights the inadequacy of solely adopting online algorithms in offline scenarios. Furthermore, the challenges associated with estimating the expectation of an exponential distribution may contribute to VDICE's potential under-performance when compared to the baseline behavioral cloning algorithm. 

\section{Conclusions}\label{sec:conclusions}

In this paper, we introduce Conditional Kernel Imitation Learning (CKIL), a simple but novel approach to imitation learning for continuous state space problems. The approach depends on finding policies that satisfy a Markov balance equation, and hence incorporates the sequential nature of the problem in a natural way. It uses conditional kernel density estimators for estimating transition densities of the MDP and the induced Markov chain and allows for behavioral policies to be represented as neural networks, thus both together making the applicability of the approach to be almost universal. Furthermore, it does not need access to a generative model or further online interaction data, does not first need reward inference, does not do distribution matching, and allows for batch processing of offline data for scalability. It also makes no assumptions about the form of the transition densities, nor of course of the reward function. The algorithm is supported by theoretical consistency results about the estimators, and furthermore, shows remarkably good empirical performance as compared to almost all state-of-the-art IL, IRL, and AIL algorithms against which comparison is meaningful on a number of continuous state OpenAI Gym benchmark environments. 

There are a number of interesting directions for future work. Since the method relies on kernel density estimation, scalability to high-dimensional problems needs to be addressed. One potential remedy involves exploring a modified version of CKDE designed to enhance computational efficiency \cite{fastCKDE}. Given that our framework is fundamentally rooted in the estimation of conditional probability densities, it readily lends itself to adaptations that can incorporate choices for substituting the CKDE, as evidenced in various alternative approaches \cite{ckdeAlternative1}, \cite{ckdeAlternative2}, \cite{ckdeAlternative3}. In fact, conditional densities themselves could have non-parametric representions. There is also scope for deeper and more detailed theoretical analysis, including non-asymptotic sample complexity bounds. Such bounds are rarely available in the current literature for IL, IRL or AIL algorithms.

\bibliography{aaai23}

\section{Appendix}\label{sec:Appendix}

\subsection{Lemma used in Theoretical Guarantees subsection \ref{subsec:Theory}}

We state the following technical lemma for convergence in probability of kernel density estimators from i.i.d samples which we used for Theorem \ref{thm1}. 

\begin{lemma}
    \cite{multivariateKernelBook} 
    Suppose $X_1, X_2, \hdots, X_n$ are i.i.d vectors with probability density $g$. Let $\hat{g}(\cdot ; H)$, as given in Eq. (\ref{eqn:MultivariateConditionalKernel}), be the kernel density estimator constructed from these samples using kernel $K$ and bandwidth matrix $H = H(n)$. Suppose the following assumptions hold.
    \begin{description}
    \item[(C1)] Each entry of $\mathcal{H}_{g}(\cdot)$ be piecewise continuous and square integrable, where $\mathcal{H}_{g}$ is the $m \times m$ \textit{Hessian-matrix} of $g$. 
    \item[(C2)] The kernel $K$, is square integrable, spherically symmetric and with a finite second order moment; 
    this means that $\int_{\mathbb{R}^m} z K(z) dz=0$ and $\int_{\mathbb{R}^m} zz^{T} K(z) d z =m_2(K)$$I_{m}$ (where $m_2(K)$ is independent of $i$, for $i \in \{1, 2, \hdots m\}$). Furthermore, $\int_{\mathbb{R}^m}K(z) = 1$.
    
    \item[(C3)] The bandwidth matrices $H=H(n)$ form a sequence of positive definite, symmetric matrices such that as $n \rightarrow \infty$, vec $H(n) \rightarrow 0$, i.e. all entries of $H(n)$ approaches $0$ and $n^{-1}|H(n)|^{-1 / 2} \rightarrow 0$, where vec is the vectorization operator which acts on a matrix by stacking its columns on top of one another.
\end{description}
    Then, $\hat{g}(\boldsymbol{x}; H)$ converges in probability to $g(\boldsymbol{x})$ for each $\boldsymbol{x}$.
    \label{thm2}
\end{lemma}

\end{document}